\PassOptionsToPackage{table,xcdraw}{xcolor}
\documentclass[letterpaper]{article} % DO NOT CHANGE THIS
\usepackage{aaai25}  % DO NOT CHANGE THIS
\usepackage{times}  % DO NOT CHANGE THIS
\usepackage{helvet}  % DO NOT CHANGE THIS
\usepackage{courier}  % DO NOT CHANGE THIS
\usepackage[hyphens]{url}  % DO NOT CHANGE THIS
\usepackage{graphicx} % DO NOT CHANGE THIS
\usepackage{natbib}  % DO NOT CHANGE THIS AND DO NOT ADD ANY OPTIONS TO IT
\usepackage{caption} % DO NOT CHANGE THIS AND DO NOT ADD ANY OPTIONS TO IT
\usepackage{algorithm}
\usepackage[commentColor=blue]{algpseudocodex}
\algrenewcommand\textproc{} % Makes algorithm not autocaptialize
\usepackage{subcaption}  % For subfigure
\usepackage{hhline}
\usepackage{multirow}
\usepackage{placeins}
\usepackage{amssymb}

\usepackage{amsthm}

\newtheorem{theorem}{Theorem}
\newtheorem{lemma}{Lemma}

\newtheorem{definition}{Definition}

\usepackage[table,xcdraw]{xcolor}
\usepackage[modulo,switch]{lineno} % module=every 5, switch = 2 columns left, right 
\usepackage{amsmath}

\nocopyright

\usepackage{newfloat}
\usepackage{listings}
\DeclareCaptionStyle{ruled}{labelfont=normalfont,labelsep=colon,strut=off} % DO NOT CHANGE THIS
\floatstyle{ruled}
\newfloat{listing}{tb}{lst}{}
\floatname{listing}{Listing}
\title{Anytime Single-Step MAPF Planning with Anytime PIBT}
\author{
    Nayesha Gandotra\equalcontrib\textsuperscript{\rm 1}, Rishi Veerapaneni\equalcontrib\textsuperscript{\rm 1}, Muhammad Suhail Saleem\textsuperscript{\rm 1}, Daniel Harabor\textsuperscript{\rm 2}, Jiaoyang Li\textsuperscript{\rm 1}, Maxim Likhachev\textsuperscript{\rm 1}
}
\affiliations{
    \textsuperscript{\rm 1} Carnegie Mellon University
    \textsuperscript{\rm 2} Monash University \\
    \{nayeshag, rveerapa, msaleem2, jiaoyanl, mlikhach\}@andrew.cmu.edu, daniel.harabor@monash.edu 
}

\begin{document}

\maketitle

\begin{abstract}
PIBT is a popular Multi-Agent Path Finding (MAPF) method at the core of many state-of-the-art MAPF methods including LaCAM, CS-PIBT, and WPPL. 
The main utility of PIBT is that it is a very fast and effective single-step MAPF solver and can return a collision-free single-step solution for hundreds of agents in less than a millisecond. However, the main drawback of PIBT is that it is extremely greedy in respect to its priorities and thus leads to poor solution quality.
Additionally, PIBT cannot use all the planning time that might be available to it and returns the first solution it finds.
We thus develop Anytime PIBT, which quickly finds a one-step solution identically to PIBT but then continuously improves the solution in an anytime manner. We prove that Anytime PIBT converges to the optimal solution given sufficient time. We experimentally validate that Anytime PIBT can rapidly improve single-step solution quality within milliseconds and even find the optimal single-step action.
% Anytime PIBT is a drop-in replacement for PIBT and therefore can be used in a wide set of algorithms. 
However, we interestingly find that improving the single-step solution quality does not have a significant effect on full-horizon solution costs. 

% PIBT is an extremely greedy 1-step MAPF solver that can plan a single step plans in milliseconds. However, PIBT suffers from poor solution costs given its greedy nature.
% Additionally, in many scenarios we may have additional time to plan. PIBT is not able to leverage this additional time.
% Thus, we develop an anytime PIBT algorithm that is able to improve its 1-step solution quality over time. 
\end{abstract}

\section{Introduction}
Efficient and collision-free navigation is a fundamental challenge for teams of robots operating in shared environments such as factory floors, warehouses, or autonomous vehicle hubs. The field of Multi-Agent Path Finding (MAPF) develops planning algorithms that enable multiple robots to move from their start locations to designated goals without colliding with or obstructing each other. 

Priority Inheritance with Backtracking (PIBT) is a popular MAPF algorithm that has gained attention due to its simplicity, speed, and effectiveness \cite{okumara2022pibt_jair}. PIBT is a one-step, greedy algorithm that quickly determines the next action for agents to execute and can generate a one-step plan for hundreds of agents in less than a millisecond. 
% Its exceptional computational speed allows it to plan paths for hundreds of agents in less than a millisecond, making it particularly well-suited for real-time applications. 
Due to its efficiency, PIBT has been integrated into various MAPF frameworks. For example, it is used for full-horizon planning in LaCAM \cite{okumura2023lacam}, collision-shielding for learned local MAPF policies in CS-PIBT \cite{veerapaneni2024improving_mapf_policies_with_search}, and as the initial solution generator under tight runtime constraints in WPPL \cite{jiang2024scaling_mapf_competition}. 
% These frameworks underscore the versatility and utility of PIBT. 

% PIBT operates by assigning priorities to agents, typically at random. It then iteratively selects the highest-priority agent whose next action has not been determined and assigns it the best available action based on a heuristic evaluation of the state it would lead to. Cells that will be occupied by higher-priority agents in the next time step, or cells occupied by static obstacles, are treated as reserved and unavailable for selection. If an agent's action displaces another lower-priority agent, the displaced agent inherits the priority of the displacing agent. Subsequently, the algorithm computes and assigns the best possible action to the displaced agent among its remaining available options. This process propagates recursively. Once the actions for all displaced agents have been determined, PIBT proceeds to the next unassigned highest-priority agent and repeats the process. This greedy, priority-based approach ensures that a valid set of actions is computed quickly.

While PIBT's greedy nature enables exceptional speed, it also introduces a key limitation: poor solution quality. 
% PIBT can compute a solution for hundreds of agents in less than a millisecond, but this speed often comes at the cost of solution quality.
% While such sub-millisecond computations are advantageous in some scenarios, there are cases where a larger planning budget is available, and higher-quality solutions are desirable. For instance, in real-time multi-agent systems that operate with a 1-step planning-execution cycle, the planner computes the next action while the agents execute the current one. This allows the planner to use the execution time as its planning budget, which is typically significantly larger than 1 ms. 
% However, 
In its current form, PIBT always computes a solution as quickly as possible (usually sub-millisecond), regardless of the available time budget which could range to hundreds or thousands of milliseconds. However, real-life applications typically have longer planning times. A common paradigm for such applications is ``anytime" planning where a planner outputs an initial solution fast and then iteratively improves it until it reaches the timeout, where it then returns the best solution found. PIBT in its current form cannot do this and cannot take advantage of additional planning time to improve its solution.
% In particular, it lacks the capability to leverage additional time to refine its solution and improve its quality. 

To this extent, we propose Anytime PIBT, which quickly finds a single-step solution identical to PIBT and then continually refines the solution improving its quality with additional planning time.
Anytime PIBT achieves this by searching through the all possible actions for agents in an efficient anytime manner. 
% A naive approach to identifying the optimal actions would require reasoning over all possible actions, which becomes computationally prohibitive, scaling as $2^N$ for $N$ agents. 
% To manage this complexity, we incorporate two key ideas.

Anytime PIBT utilizes two main insights. First, we interpret the single-step MAPF problem as a recursive problem where we can employ a Depth-First Search (DFS). This allows us to explore through all actions while saving intermediate solutions and pruning.
% First, we use the cost of the current solution to prune branches in the search tree that cannot yield better solutions, significantly reducing the number of possibilities to consider. 
Second, we exploit the semi-independent structure of MAPF problems. At any given time, not all agents interact or affect one another's actions. 
% Only agents in close proximity influence each other's ability to execute their best actions. 
By identifying these interacting agents and grouping them, we break the joint action space into smaller, independent subspaces, making Anytime PIBT more efficient.

% The framework starts by querying vanilla PIBT to compute an initial solution while incorporating additional bookkeeping to identify groups of interacting agents. These groups are assumed to be potentially independent and are refined individually. Anytime-PIBT is applied to these groups to compute their optimal actions, continuing until the optimal actions are determined or the available time budget is exhausted. If, during this refinement, agents from different groups are found to interact, the group structure is dynamically updated, and the planner is re-invoked accordingly.

We theoretically prove that Anytime-PIBT algorithm converges to the optimal single-step solution given sufficient planning time. Additionally, we empirically evaluate its performance on both small and large problem instances (ranging from 20 to 1000 agents) and find that it can consistently improve the single-step cost. However, Anytime PIBT does not significantly improve full-horizon costs compared to PIBT when used with LaCAM and LaCAM*. Still, given the popularity of PIBT, we believe that Anytime PIBT has potential widespread use and promising future directions.
% , demonstrating its capability to improve the quality of each action as more planning time is allocated.

\section{Preliminaries and Related Work}
% We focus on anytime MAPF algorithms and applications of PIBT.
% We will do a brief overview of the main families of heuristic search MAPF solvers and then focus on anytime methods and PIBT.

% \section{Terminologies and Preliminary}

Multi-Agent Path Finding (MAPF) involves finding collision-free paths for a set of $N$ agents, denoted as ${i = 1, \dots, N}$, where each agent must travel from its start location $s_i^{\text{start}}$ to its goal location $s_i^{\text{goal}}$. In the standard 2D MAPF setup, the environment is discretized into grid cells. Agents can move to adjacent cells in any cardinal direction or remain stationary in their current cell. A valid solution consists of a set of paths $\Pi = \{ \pi_1, ..., \pi_N \}$ where $\pi^0_i = s_i^{start}$, $\pi^{T}_i = s_i^{goal}$ with $T$ representing the maximum timestep of all agents' paths. To ensure validity, the solution must avoid vertex collisions (when two agents occupy the same cell at the same timestep, i.e., $\pi^t_i = \pi^t_j$ for $i \neq j$) and edge collisions (when two agents swap positions between consecutive timesteps, i.e., $\pi^t_i = \pi^{t+1}_j \wedge \pi^{t+1}_i=\pi^t_j$). The standard objective in optimal MAPF is to find a solution $\Pi$ that minimizes the total cost $|\Pi^{0:T}| = \sum_{i=1}^N |\pi_i^{0:T}| = \sum_{i=1}^N \sum_{t=0}^{T-1} c(s_i^t,s_i^{t+1})$. In this work, we assume every action to be of unit cost, $c(s_i^t,s_i^{t+1})=1$, except when an agent remains at its goal (in which case the cost is zero). 

While the above describes the standard formulation of the full-horizon MAPF problem, this work focuses on a simplified variant: the single-step MAPF problem. Here, the objective is to determine the next single actions for agents that minimizes the total cost for all agents to reach their goals, under the assumption that after executing the first action, each agent follows its individual optimal path to the goal. The latter part ignores potential interactions with other agents and therefore is just each agent's individual optimal path to the goal. We note that all performant 2D MAPF methods compute a backward Dijkstra's for each agent where $h^*_i(s) = c^*(s,s_i^{goal})$. Thus our objective is to compute the first action for all agents that minimizes $\Pi^{0:1} = \sum_{i=1}^N (c(s_i^0,s_i^{1}) + h^*_i(s_i^1))$.
% $\pi_i^{0}$

\begin{figure*}[t!]
    \centering
    \includegraphics[width=0.85\linewidth]{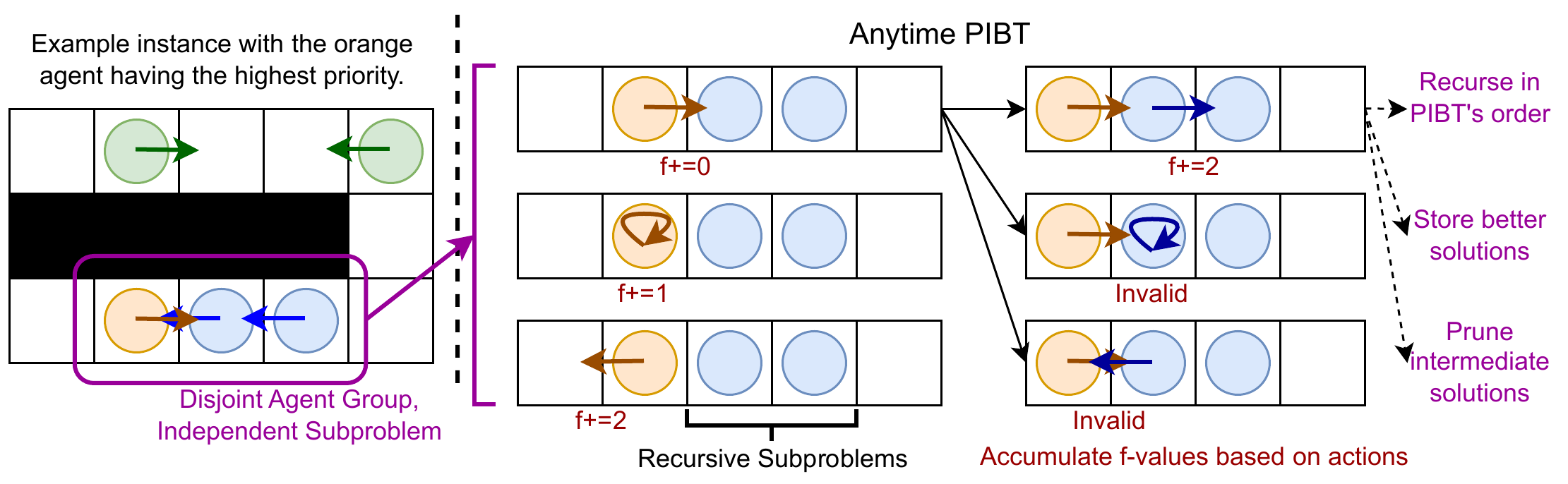}
    \caption{The left example shows six agents with preferred actions denoted in arrows. The orange agent with higher priority would push back the blue agents, when an optimal single-step plan would have the orange agent move back instead. Anytime PIBT first detects initial disjoint agent groups through an initial modified PIBT call. Then Anytime PIBT solves each group by recursing through possible actions and agents in PIBT's order. It stores encountered better solutions and prunes intermediate solutions based on the accumulated penalty.}
    \label{fig:overview}
    \vspace*{-1em}
\end{figure*}

\subsection{PIBT} \label{sec:pibt}
PIBT is a single-step planner that starts out by assigning each agent a priority. It then sequentially plans agents (in descending order of priority) with each agent reserving its next action/location. 
% The key idea is that if an agent $R_i$ reserves a location that is occupied by another agent $R_j$, then $R_j$ needs to plan next to try to make way for $R_i$. 
If an agent $R_i$ reserves a location that collides with the single-step path of a previously planned (i.e. higher priority) agent, $R_i$ is not allowed to try that action and must try its next preferred action. If $R_i$ instead moves into a lower priority agent $R_k$, the other lower priority agent $R_k$ is required to plan to make way for $R_i$ (hence ``inherits" the higher priority). This priority inheritance repeats until an action is found. If $R_k$ cannot compute a valid action, it then chooses to stay in place and requires $R_i$ to reserve a different action. This ``backtracking" behavior enables PIBT to remain effective in congested scenarios where agents must try out different actions. 

PIBT is greedy in respect to the provided single-agent heuristic and agent priorities, e.g., the highest priority agent always prefers to minimize its heuristic, even if that would delay all the other agents. Thus PIBT is not single-step optimal, i.e., does not minimize $|\Pi^{0:1}|$ and can return arbitrarily bad solutions.

\subsection{PIBT in Other Methods}
PIBT is popular as it is an extremely fast single-step MAPF solver ($<1$ millisecond for hundreds of agents). In particular, the LaCAM solvers, i.e., LaCAM~\cite{okumura2023lacam}, LaCAM*~\cite{okumara2023lacam_star}, Engineering LaCAM*~\cite{okumara2024engineering_lacam}, search over the joint configuration space and leverage PIBT as their fast joint configuration successor generator. The different methods search over the joint space in different manners with LaCAM using a Depth-First Search (DFS), LaCAM* using anytime A*, and Engineering LaCAM* employing many tools including parallelism. These methods all rely on PIBT for its speed.

PIBT's collision-free single-step behavior has also been used in CS-PIBT as a smart ``collision shield" for post-processing learnt local MAPF policy predictions \cite{veerapaneni2024improving_mapf_policies_with_search}. Since single-step learnt policies could make errors and predict actions that lead to collisions, they propose using PIBT informed by the model predictions to produce collision-free actions.
% Instead of using PIBT where agents try to minimize their individual heuristics, the learnt local MAPF policy predictions as the agent's action preferences and show how using CS-PIBT can improve performance compared to a naive collision shield.

Traffic Flow~\cite{chen2024traffic_flow} improves PIBT's performance in Lifelong MAPF (where agents are continually assigned new goals) by computing heuristics online informed by agent congestion. This decreases the greedy single-step behavior of PIBT and was shown to improve performance. 

% Windowed Parallel PIBT-LNS (WPPL)~\cite{jiang2024scaling_mapf_competition} used PIBT to generate an initial windowed solution for 1000s of agents in severe congestion and a 1-second timeout, and then used parallel LNS to improve the solution quality with the remaining time.

% have extremely high success rates and fast planning time. These methods search over the joint-configuration space in different manners (e.g. LaCAM uses Depth-First Search (DFS), LaCAM* uses A*) and require an extremely fast single-step MAPF solver to be effective, and thus use PIBT.

% Conflict Based Search \cite{sharon2015cbs} is another foundational method which uses a low-level search to find single-agent paths that obey constraints, and a high-level search that searches over constraints that are applied to avoid conflicts. CBS has a wide variety of follow-up methods including BCBS, ECBS \cite{barer2014suboptimalecbs}, EECBS \cite{li2021eecbs}, and W-EECBS \cite{effectiveCBS}.

\subsection{Anytime MAPF Algorithms}
Anytime algorithms find initial solutions very fast and then improve them over time, allowing them to be queried at ``anytime" where they return the best solution found so far. Anytime Focal search~\cite{cohen2018anytime_focal}  replaced the high-level focal search of BCBS~\cite{barer2014suboptimalecbs} with an anytime focal search, allowing it to improve the solution quality over time. However, it is unclear how such anytime conflict-based methods would scale to 100s of agents where finding an initial solution itself could take seconds.
% In theory, Anytime BCBS search would reach the optimal solution, but in practice it seldom does.

A more general popular approach for anytime behavior is to use Large Neighborhood Search (LNS)~\cite{shaw1998lns} in MAPF. LNS is a broad local search method that tries to find better solutions by taking an existing solution, destroying part of it, and then repairing it. MAPF-LNS~\cite{li2021mapf-lns} and MAPF-LNS2~\cite{li2022mapf-lns2} explore different algorithms for generating the initial MAPF solution, destroying, and repairing and showed significant improvement over Anytime EECBS. Using LNS however does not have any guarantees on reaching the optimal solution.

Anytime algorithms are particularly useful when the time cutoff is very small. The Robot Runners competition \cite{chan2024_league_robot_runners_competition} required planning for 100s-1000s of agents in one second. The winning team developed Windowed Parallel PIBT-LNS (WPPL) which used PIBT on weighted graphs to obtain an initial windowed solution (of length 5-15 depending on the instance) and then parallel LNS to improve the solution \cite{jiang2024scaling_mapf_competition}. Planning Interleaved with Execution \cite{zhang2024pie} improves an initial solution (found by LaCAM*) using LNS while simultaneously executing the current best-found path. 

% PIBT is popular as it is an extremely fast single-step MAPF planner ($<1$ millisecond for hundreds of agents) as seen in its use in the LaCAM methods and WPPL. Additionally, PIBT has been used as a one-step post-processing technique to resolve collisions from predictions made by learnt local MAPF policies \cite{veerapaneni2024improving_mapf_policies_with_search}. However, PIBT is almost unnecessarily fast and does not use any additional planning time that could be remaining.

% Our contribution is thus to create an anytime single-step planner, Anytime PIBT. Anytime PIBT is an anytime single-step planner that finds an identical solution to PIBT initially and then continues to improve the solution, eventually finding the optimal single-step solution given enough planning time. This eventually optimal nature distinguishes itself from prior works using LNS (e.g. WPPL) which have no guarantees on eventual solution quality.

\subsection{Disjoint Agent Groups}
One core aspect of our proposed Anytime PIBT is planning over groups of agents rather than all agents together. This idea is inspired by two main works that utilize this idea.

Operator Decomposition (OD) \cite{standley2010operater_decomposition} uses independence detection to iteratively divide a MAPF problem of all agents into a set of smaller MAPF problems $M<N$ where each smaller problem contains a set of conflicting agents. The WinC-MAPF framework formalizes this notion and defines Disjoint Agent Groups (DJAGs) where agents in different disjoint agent groups do not affect each other's solutions \cite{veerapaneni2024winc_mapf}.

% Operator Decomposition (OD) \cite{standley2010operater_decomposition} approaches an MAPF problem by planning in the joint space, but generating intermediate successors by moving only a single agent at a time. OD uses an A* search on these intermediate successor nodes to successively select and generate nodes until all agents have reached the goal.

% Standley noted however that this algorithm is still exponential in the number of agents, and could be significantly smaller if we also searched ``independent subproblems'' of disjoint sets of agents that did not interfere with each other. The WinC-MAPF framework formalizes this notion and defines Disjoint Agent Groups where agents in different disjoint agent groups do not affect each other's solutions \cite{veerapaneni2024winc_mapf}.

\begin{definition}[Disjoint Agent Groups]
    Given a configuration transition $s_{1:N} \rightarrow s'_{1:N}$, and a set of disjoint agent groups $\{Gr^i\}$, we have the property that for each agent $R_j$ with transition $s_j \rightarrow s_j'$ in disjoint agent group $Gr^i$, there cannot exist another agent in a different group $Gr^k$ that blocked $R_j$ from picking a better path.
\end{definition}

% Planning optimal (single-step) paths for disjoint agent groups will lead to an overall optimal solution as agents in disjoint agent groups will not have paths that interact with each other (by definition). 
The main implication of DJAGs is that we can plan for each DJAG separately (as agents in different groups do not affect each other).
We use this idea in Anytime PIBT and detect and solve DJAGs, which significantly improves the anytime behavior.

% Additionally, the optimal solution for all solutions has the property that each DJAG in the optimal solution has its optimal solution. A formal proof is found in \cite{veerapaneni2024winc_mapf}, the intuition is that if a DJAG does not have its optimal solution, we can swap out its current solution with the optimal solution (since it is disjoint, it would not affect other agent groups, if it did affect other agents then that violates the definition of DJAG). 

Additionally, finding the optimal solution for each DJAG ensures that we find the global optimal solution (as the solution for each DJAG captures interactions between agents internally, but agents in different DJAGs do not interact with each other). 
% The intuition is fairly straightforward; Since each DJAG is optimally planned, and agents in each DJAG do not block each other, there is no way the cost can be decreased.
% The proof is easier via the contrapositive; if we do not find the global solution, then there exists a DJAG
Thus, individually finding, optimizing, and recalculating DJAGs (if agents have new conflicts when replanning) guarantees eventually finding the optimal solution (see Theorem \ref{thm:optimality}).

\begin{algorithm}[t!]
\caption{Anytime PIBT}
\label{alg:alg-anytime-pibt}
\textbf{Parameters}: Current states $s^t_{1:N}$, Agent Priorities $AP_{1:N}$, Timeout $T_{out}$ \\ \noindent
\textbf{Output}: $\pi^{t+1}_{1:N}$
\begin{algorithmic}[1] %[1] enables line numbers

\Procedure{AnytimePIBT}{$s^t_{1:N}$, $AP_{1:N}$, $T_{out}$}
    \State $\pi^{t+1}_{1:N} =$ Null, GroupQ = $\emptyset$
    \State PIBTwithGrouping($s^{1:N}$, $AP_{1:N}$) \Comment{Populates $\pi^{t+1}_{1:N}$ and GroupQ} \label{line:pibt-with-groups}
    \State $\pi^{*t+1}_{1:N} = \pi^{t+1}_{1:N}$ \Comment{Current best solution}
    % \For{$Gr$ $\in$ Groups} \Comment{$Gr$ contains a set of agents}
    \While{!GroupQ.empty() and not timeout}
        \State $Gr$ = GroupQ.pop() \Comment{Group class, see Alg \ref{alg:alg-anytime-helper}}
        % \State $T_r = $ GetRemainingTime()
        \State $T_{group} = \text{TimePerGroup}(|Gr.AoP|, \text{GroupQ})$ \label{line:group-timeout}
        % \State $$
        % \State $T_{group} = T_{out}/|\text{Groups}|$
        \State $\pi^{t+1}_{\forall i \in Gr.AoP} = $ Null \Comment{Clear for replanning}
        \State AnytimePIBT-R($Gr$, $Gr.AoP$, Null, $s^t_{1:N}$, 0, $T_{group}$) \Comment{Alg \ref{alg:alg-anytime-helper}, updates $\pi^{*t+1}_{i \in Gr}$, earlyExit, NewGroup} \label{line:main-anytime-helper}
        \If{NewGroup $\neq Gr$} \label{line:main-anytime-new-group}
            \State GroupQ.removeNotDisjointWith(NewGroup)
            \State GroupQ.push(NewGroup)
        \ElsIf{earlyExit} \Comment{Did not finish group}
            \State GroupQ.push($Gr$) \Comment{We can revisit this later if given enough time}
        \EndIf
    \EndWhile
    \State \textbf{return} $\pi^{*t+1}_{1:N}$
\EndProcedure
\Statex

% \Function{CreateGroup}{Agents}
%     \State $Gr.F_{best} = \infty$
%     \State $Gr.AoP = $ Agents
% \EndFunction
% \Statex

% \Procedure{AnytimePIBT}{$s^{1:N}$, $AP^{1:N}$, $T_{out}$}
%     \State $\pi^{t+1}_{1:N} =$ Null, Groups = $\emptyset$
%     \State PIBTwithGrouping($s^{1:N}$, $AP^{1:N}$) \Comment{Populates $\pi^{t+1}_{1:N}$ and Groups} \label{line:pibt-with-groups}
%     \For{$Gr$ $\in$ Groups} \Comment{$Gr$ contains a set of agents}
%         \State $T_{group} = \text{TimePerGroup}(T_{out})$ \label{line:group-timeout}
%         % \State $$
%         % \State $T_{group} = T_{out}/|\text{Groups}|$
%         \State $\forall i \in Gr$, $\pi^{t+1}_i = $ Null \Comment{Clear for replanning}
%         \State AnytimePIBT-R($Gr$, Null, $s^t_{1:N}$, 0, $T_{group}$) \Comment{Alg \ref{alg:alg-anytime-helper}, updates $\pi^t_{i \in Gr}$ and could populate NewGroup} \label{line:main-anytime-helper}
%         \If{NewGroup $\neq \emptyset$} \label{line:main-anytime-new-group}
%             \State Groups.removeNonDisjointWith(NewGroup)
%             \State Groups.append(NewGroup)
%         \EndIf
%     \EndFor
%     \State \textbf{return} $\pi^{t+1}_{1:N}$
% \EndProcedure
% \Statex

\Function{PIBTwithGrouping}{$s^t_{1:N}$, $AP_{1:N}$}
    % \State Reserved = $\emptyset$
    % \State Moves = dictionary()
    \For{agent $k \in $ argsort($AP_{1:N}$)} \label{line:priorities}
        % \If {$k \notin M_c$.keys()} \Comment{If no move planned}
        \If{$\pi^{t+1}_k$ is Null} \Comment{If no move planned}
            \State PIBT-Gr($k$, $s^t_{1:N}$) \Comment{Alg \ref{alg:alg-pibt-helper}, updates $\pi^t_{1:N}$}
        \EndIf
    \EndFor
    % \State \textbf{return} Moves
\EndFunction

\Statex
\Function{TimePerGroup}{$K$, GroupQ}
    \State totalAgentsInGroups = $\sum_{Gr^i \in \text{GroupQ}} |Gr^i.AoP|$
    % $|ag \in Gr^i, \forall Gr^i \in Groups|$
    \State \textbf{return} getTimeLeft()$*K/$totalAgentsInGroups
\EndFunction
\end{algorithmic}
\end{algorithm}

\section{Anytime PIBT}
Anytime PIBT utilizes two main insights.
First, as depicted in the middle of Figure \ref{fig:overview}, we view the single-step MAPF as a recursive problem; planning for $N$ agents means assigning a location for one agent and then planning for the rest $N-1$ agents. Given $N$ agents, each with 5 actions, this means we could consider all $5^N$ options through our recursive tree. We then save intermediate solutions in the tree, and can prune out intermediate branches whose solution is worse than our best found so far.
We use PIBT's priority inheritance and backtracking to order our recursive calls. Thus, Anytime-PIBT at its core is a DFS through the action tree that employs standard solution saving and pruning.

However, this recursive structure disregards the semi-independence of agents and requires recursing through / resolving identical subproblems repeatedly. Thus, our second insight is to use the concept of disjoint agent groups from the WinC-MAPF framework \cite{veerapaneni2024winc_mapf} and decompose the single-step MAPF problem into smaller disjoint agent groups (as seen in left of Figure \ref{fig:overview}). This enables us to scale to more agents and have better anytime performance.
Given these two insights (high-level DFS and grouping), we now describe Anytime PIBT using the pseudocode in Algs \ref{alg:alg-anytime-pibt}, \ref{alg:alg-pibt-helper}, and \ref{alg:alg-anytime-helper}. 
% Since Anytime PIBT only deals with a single timestep, we omit time indices from the superscript and instead use it to denote agents.

% Anytime PIBT is sufficiently straightward to describe that we believe following the psuedocode is the cleanest way of explaining the method given the high-level DFS intuition.

\subsubsection{Initial PIBT with Grouping Call}
Anytime PIBT (Alg \ref{alg:alg-anytime-pibt}) starts with a modified PIBT call (line \ref{line:pibt-with-groups}) that keeps track of agents that interact with each other and define disjoint agent groups. The only modification to regular PIBT is that the recursive helper PIBT call (PIBT-Gr, Alg \ref{alg:alg-pibt-helper}) groups agent $k$ when its action is blocked by a higher priority agent $j$ (line \ref{line:pibt-vertex-edge-conflict}) or when agent $k$ bumps into another agent $j$ (line \ref{line:pibt-bump-priority}). Agents $k$ and $j$ with such interactions belong to the same group as they affect each other's ability to pick their best single-step path. On the flip side, agents without such interactions are not in the same agent group. 

Note that our grouping scheme might contain extra agents (e.g., $k$ bumping into $j$ does not necessarily mean $j$ blocks $k$ if they both move in the same direction). However, this is still a valid disjoint agent group and would be fast to plan for later on as these agents move on their optimal actions (and thus we would find the optimal solution immediately and would spend negligible time planning for these agents).

The ``Group" function is a simple union that merges agents that have been grouped. We highlight that grouping is just bookkeeping and adds negligible time. Thus after the end of the initial PIBT call, we have an initial solution identical to PIBT, as well as disjoint groups of agents. In Figure \ref{fig:overview}, this would lead the bottom three agents to be in a single group and the green agents to not be in any group.

\begin{algorithm}[t!]
\caption{PIBT with Grouping Recursive Function, \textcolor{red}{red} denotes modifications to regular PIBT}
\label{alg:alg-pibt-helper}
\textbf{Parameters}: Agent $k$, Current states $s^t_{1:N}$ \\ \noindent
\textbf{Globals and Side Effects}: $\pi^{t+1}_{1:N}$ being populated, Group() updates disjoint agent groups %\\ \noindent
% \textbf{Side effects}: Group() updates disjoint agent groups
% \textbf{Output}: Collision free actions $a^{1:N}$
\begin{algorithmic}[1] %[1] enables line numbers
\Procedure{PIBT-Gr}{$k$, $s^t_{1:N}$}
    % \State Global $M_c$ \Comment{Moves dictionary being populated}
    \For{$s^{t+1}_k \in \text{Neighbors}(s^t_k)$} \Comment{Sort by $c(s^t_k, s^{t+1}_k) + h^*(s^{t+1}_k)$}
        \If {$s^{t+1}_k$ is invalid} \Comment{Ignore obstacles}
            \State \textbf{continue}
        % \ElsIf {$\exists j$ s.t. $\pi^{t+1}_j = s^{t+1}_k$} \Comment{Vertex conflict} \label{line:pibt-vertex-conflict}
        %     \State \textcolor{red}{Group($k, j$)}
        %     \State \textbf{continue}
        % \ElsIf{$\exists j$ s.t. $\pi^{t+1}_j=s^t_k \wedge  \pi^t_j=s^{t+1}_k$} \label{line:pibt-edge-conflict} \Comment{Edge Conflict}
        %     \State \textcolor{red}{Group($k, j$)}
        %     \State \textbf{continue}
        % \EndIf
        \EndIf
        \If {$\exists j$ s.t. $\pi^{t+1}_j = s^{t+1}_k$ or $\pi^{t+1}_j=s^t_k \wedge  \pi^t_j=s^{t+1}_k$} \Comment{Vertex or Edge collision} \label{line:pibt-vertex-edge-conflict}
            \State \textcolor{red}{Group($k, j$)}
            \State \textbf{continue}
        % \ElsIf{$\exists j$ s.t. $\pi^{t+1}_j=s^t_k \wedge  \pi^t_j=s^{t+1}_k$} \label{line:pibt-edge-conflict} \Comment{Edge Conflict}
        %     \State \textcolor{red}{Group($k, j$)}
        %     \State \textbf{continue}
        \EndIf
        \State $\pi^{t+1}_k = s^{t+1}_k$
        \If {$\exists$ agent $j \neq k$ with $s^t_j = s^{t+1}_k$} \label{line:pibt-bump-priority}
            \State \textcolor{red}{Group($k, j$)}
            \If {PIBT-Gr($j$, $s^t_{1:N}$)}
                \State \textbf{return} Success
            \EndIf
            \State $\pi^{t+1}_k=$Null
        \Else
            \State \textbf{return} Success
        \EndIf
    \EndFor
    \State \textbf{return} Failure
\EndProcedure
\end{algorithmic}
\end{algorithm}

\begin{algorithm}[t!]
\caption{Anytime PIBT Recursive Function}
\label{alg:alg-anytime-helper}
% \textbf{Parameters}: Agents to Plan $AoP$, Agent $k$, Current states $s^t_{1:N}$, Current $F_c$, Timeout $T_{out}$ \\ \noindent
% \textbf{Globals and Side Effects}: Preplanned $\pi^{t:t+1}_{\forall i \notin AoP}$, updates Paths $\pi^{t+1}_{\forall i \in AoP}$, Best solution $\pi^{*t+1}_{1:N}$, Best f-value $F_{best}$, ``earlyExit" (default True), and ``NewGroup" (defaults $AoP$) %\\ \noindent
% \textbf{Side Effects}: Updates a ``earlyExit" variable (default True), Group() also updates a ``NewGroup" variable
% \textbf{Output}: Collision free actions $a^{1:N}$

\textbf{Parameters}: Group $Gr$, Current States $s^t_{1:N}$, Timeout $T_{out}$ \\ \noindent
\textbf{Group Class:} $Gr$ has agents $Gr.AoP$ and best f-value $Gr.F_{b}$ (default $\infty$) \\ \noindent
\textbf{Globals and Side Effects}: Preplanned $\pi^{t:t+1}_{\forall i \notin Gr.AoP}$, updates Paths $\pi^{t+1}_{\forall i \in Gr.AoP}$, Best solution $\pi^{*t+1}_{1:N}$, ``earlyExit" (default True), and ``NewGroup" (defaults $Gr.AoP$) 

\begin{algorithmic}[1] %[1] enables line numbers
% \Procedure{AnytimePIBT-R}{$Gr, s^t_{1:N}, T_{out}$}
%     \State AnytimePIBT-RH($Gr$, $Gr.AoP$, Null, $s^t_{1:N}, 0, T_{out}$)
% \EndProcedure
% \Statex
\Procedure{AnytimePIBT-R}{$Gr$, Current $AoP$, Agent $k$, $s^t_{1:N}$, Accumulated f-costs $F_c$, $T_{out}$} % \Comment{$F_c$ is accumulated f-costs}
% , Best $F_b$, Best Moves $M_b$,
    % \State $h_{best} = \min_{s'^k \in \text{Neighbors}(s^k)} s'^k$ 
    % \State Global $M_c, M_{best}, F_{best}$ \Comment{Moves, best moves and f}
    \If{$k$ is Null} \Comment{If no particular agent to plan}
        \State $k = AoP$.top() \Comment{Pick from AoP by priority}
    \EndIf
    \State $AoP = AoP \setminus k$ 
    \Comment{Less agents for recursive call}
    \For{$s^{t+1}_k \in \text{Neighbors}(s^t_k)$} \Comment{Sort by $c(s^t_k, s^{t+1}_k) + h^*(s^{t+1}_k)$} \label{lines:anytime-all-actions}
        \If{getCurrentPlanningTime() $> T_{out}$}  \label{lines:anytime-timeout}
            \State \textbf{return} \Comment{Time cutoff, set earlyExit to True}
        \EndIf
        \If {$s^{t+1}_k$ is invalid} \label{lines:anytime-collisions}
            \State \textbf{continue}
        \EndIf
        % \ElsIf {$\exists j$ s.t. $\pi^{t+1}_j = s^{t+1}_k$} \label{lines:anytime-vertex-conflict}
        %     \State Group($k, j$) \Comment{Updates NewGroup} \label{lines:anytime-group-vertex}
        %     \State \textbf{continue}
        % \ElsIf{$\exists j$ s.t. $\pi^{t+1}_j=s^t_k \wedge  \pi^t_j=s^{t+1}_k$} \label{lines:anytime-edge-conflict} 
        %     \State Group($k, j$) \Comment{Updates NewGroup} \label{lines:anytime-group-edge}
        %     \State \textbf{continue}
        \If {$\exists j$ s.t. $\pi^{t+1}_j = s^{t+1}_k$ or $\pi^{t+1}_j=s^t_k \wedge  \pi^t_j=s^{t+1}_k$} \Comment{Vertex or Edge collision} \label{lines:anytime-vertex-edge-conflict}
            \State Group($k, j$) \Comment{Updates NewGroup}
            \State \textbf{continue}
        \EndIf
        % \ElsIf {$s'$ is invalid} \Comment{Ignore obstacles} \label{lines:anytime-collisions}
        %     \State \textbf{continue}
        % % \EndIf
        % \ElsIf {$s^{t+1}_k \in RS$ or $(s^{t+1}_k, s^t_k) \in RS$} \label{lines:anytime-reserved-collisions}  \Comment{Vertex or Edge Collision}
        %     % \State NewGroup = Group(NewGroup, AgentAt($s$))
        %     % \State Group($k$, AgentAt($s^{t+1}_k$)) 
        %     \State Group($k$, AgentReserving($s^{t+1}_k$ or $(s^t_k,s^{t+1}_k)$)) \label{lines:anytime-group1}
        %     \State \textbf{continue}
        % \EndIf
        \State $F_{next}$ = $F_c + c(s^t_k, s^{t+1}_k) + h^*(s^{t+1}_k)$ \label{lines:anytime-update-f}
        \If {$F_{next} \geq Gr.{F_{b}}$} \Comment{Prune if not better} \label{lines:anytime-prune-f}
            % \State \textbf{continue}
            \State \textbf{return} 
            % \Comment{Note break not continue}
        \EndIf
        % \If{$RA$ is $\emptyset$} \Comment{If planned for all agents} \label{lines:anytime-save-better}
        %     \State $M_c[k] = a$
        %     \State $F_{best}, M_{best} = F_{next}, M_c$ \Comment{Update best}
        %     \State \textbf{continue}
        % \EndIf
        \State $\pi^{t+1}_k = s^{t+1}_k$ \label{lines:anytime-reserve-actions}
        \If{$AoP$ is $\emptyset$}
        \Comment{If all agents planned}
            \State $Gr.F_{b}  = F_{next};\,  \pi^{*t+1}_{\forall i \in Gr.AoP} = \pi^{t+1}_{\forall i \in Gr.AoP}$\label{lines:anytime-save-better}
            % \Comment{Update best} \label{lines:anytime-save-better}
        \Else
            \If {$\exists$ agent $j \neq k$ with $s^t_j = s^{t+1}_k$} \label{lines:anytime-lower-agent} 
                \State Group($k, j$) \Comment{Updates NewGroup}
                \State nextAgent = $j$ \Comment{Priority inheritance}
            \Else \label{lines:anytime-next-agent}
                \State nextAgent = Null
            \EndIf
            \State AnytimePIBT-R($Gr$, $AoP$, nextAgent, $s^t_{1:N}$, $F_{next}$, $T_{out}$) \label{lines:anytime-recursive}
        \EndIf
        \State $\pi^{t+1}_k=$Null

        % \ElsIf{$AoP$ is $\emptyset$} \Comment{If planned for all agents} \label{lines:anytime-save-better}
        %     \State $M_c[k] = a$
        %     \State $F_{best}, M_{best} = F_{next}, M_c$ \Comment{Update best}
        %     \State $M_c[k] = \emptyset$
        %     \State \textbf{continue}
        % \Else \label{lines:anytime-next-agent}
        %     \State nextAgent = Null %$RA$.top()
        % \EndIf
        % %     \State PIBT-H($j$, $a^{1:N}_{1:5}$, $s^{1:N}$, $RS$, Moves)\label{lines:recusive-call}
        % %     \State $M_c$[$k$]$=\emptyset$, $RS \minuseqB \{s', (s^k,s')\}$)
        % % \Else
        % %     \State agent $j$ = $RA$.top()
        % \State $M_c$[$k$]$=a$, $RS \pluseqB \{s',(s^k,s')\}$) \label{lines:anytime-reserve-actions}
        % \State AnytimePIBT-H($AoP$, nextAgent, $s^{1:N}$, $RS$, $F_{next}$, $M_c$, $F_b$, $M_b$, $T_{out}$) \label{lines:anytime-recursive}
        % % \State AnytimePIBT-H(nextAgent, $RA\setminus j$, $s^{1:N}$, $RS$, $F_{next}$, $M_c$, $F_b$, $M_b$, $T_{out}$) \label{lines:anytime-recursive}
        % \State $M_c$[$k$]$=\emptyset$, $RS \minuseqB \{s', (s^k,s')\}$)
        % % \EndIf
    \EndFor
    % \State \textbf{return}
\EndProcedure
\end{algorithmic}
\end{algorithm}

\subsection{Solving a Group with Anytime PIBT}
After generating the groups, we then iterate through each group and call AnytimePIBT-R (Alg \ref{alg:alg-anytime-helper}) on the group. AnytimePIBT-R at its core is a DFS that goes through all agents in the group and their possible actions. We first describe it without the notion of grouping, e.g., imagine that the group with corresponding \underline{a}gents t\underline{o} \underline{p}lan ($AoP$) contains all agents. We will revisit grouping afterwards.

Similar to PIBT-Gr, AnytimePIBT-R starts by going through the states for agent $k$ (line \ref{lines:anytime-all-actions}) and ignores invalid states (line \ref{lines:anytime-collisions}) or states that conflict with previously planned agents (line \ref{lines:anytime-vertex-edge-conflict}).

% One main difference is that it now adds in the $c(s^k,s') + h^*(s')$ of the chosen state and updates the f-value of the actions made so far (line \ref{lines:anytime-update-f}). 
One main difference is that it accumulates the f-value of all planned agents by adding the $c(s^k,s') + h^*(s')$ of the chosen state to the current accumulated sum (line \ref{lines:anytime-update-f}). 
If this f-value is greater or equal than the current best solution's $F_b$, we can immediately prune the rest of the DFS branch as we are guaranteed that moving future agents will only increase the f-value (line \ref{lines:anytime-prune-f}). One subtlety is that we sorted $s'$ by increasing $c(s^k, s') + h^*(s')$, and thus can prune out the following neighbor states $s''$ with larger $c(s^k, s'') + h^*(s'')$, allowing us to \textit{return} as opposed to \textit{continue}.

If we do not prune, we update $\pi^{t+1}_k$ accordingly (line \ref{lines:anytime-reserve-actions}). 
If we have finished planning for all agents, we are guaranteed that the accumulated f-value is smaller than $F_b$ (since it was not pruned before) and we thus update the best solution and f-value (line \ref{lines:anytime-save-better}).

If we have not finished planning and bump into another agent, we choose that to be the next agent to plan (line \ref{lines:anytime-lower-agent}) following PIBT's priority inheritance logic. 
We then recurse (line \ref{lines:anytime-recursive}) which goes through the other agents and their action. Note if we did encounter another agent, the recursive call will pick the next (highest priority) agent in $AoP$.

Finally, we reset $\pi^{t+1}_k$ and proceed to consider the next action. We repeat this logic for all actions and therefore consider all possibilities of movements for each agent.

In summary, Anytime PIBT is a DFS that systematically goes through agents and their actions while keeping track of f-values and saving or pruning solutions accordingly. The ability to stop the DFS at anytime (line \ref{lines:anytime-timeout}) allows Anytime PIBT to return the best solution found at any time.

\subsection{Updating Groups on the Fly}
Anytime PIBT described so far searches over actions for agents in the group while having the rest of the agents move following their original PIBT solution. Thus, an agent $k \in Gr.AoP$ (where $Gr.AoP$ are the agents in the $k$'s group $Gr$) cannot conflict with an agent $j \notin Gr.AoP$ inside AnytimePIBT-R, as denoted by the preplanned $\pi^{t:t+1}_{\forall i \notin Gr.AoP}$ in Alg \ref{alg:alg-anytime-helper} ``Globals". Put in other words, when planning for $Gr.AoP$, the agents not in $Gr.AoP$ are assumed to be following their existing paths and must be avoided. 
% (as that will cause a vertex or edge conflict with $\pi^{t:t+1}_j$ 
% Thus, agents not in the group $Gr$ must prepopulate $RS$ (Alg \ref{alg:alg-anytime-pibt}, line \ref{line:main-anytime-helper}) so that when Anytime PIBT plans the agents in $Gr$, those agents will not conflict.

This can introduce suboptimality issues. For example it is possible that an agent $k$ and $j$ are initially in different groups. However, during the Anytime PIBT call to the group containing $k$, $k$ could consider an action that conflicts with $\pi^{t:t+1}_j$. This action would not be pursued (Alg \ref{alg:alg-anytime-helper}, line \ref{lines:anytime-vertex-edge-conflict}).
% violating our definition of disjoint agent groups. 
Likewise, planning for $j$ could have the same issue. This means that our planning procedure is missing possible actions of considering both agent $k$ and $j$ moving informed by each other. On the flip side, if $k$ and $j$ were in the same group, a joint movement would not be pruned and instead be considered through the recursive logic.
Thus, to maintain optimality, we must update our disjoint agent groups so that $k$ and $j$ are grouped together. Note this in turn means that their respective groups should be merged too.

% However, it is possible that during Anytime PIBT, an agent $k$ considers an action that interacts with different agents $j$ not in its group (as the original group was determined by PIBT's actions, not Anytime PIBT's action). This would appear in Alg \ref{alg:alg-anytime-helper}, line \ref{lines:anytime-reserved-collisions} and prevent.

% This interaction means that if we want to eventually find the optimal solution across all agents (and thus all agent groups), we cannot keep $j$'s action constant when planning for $k$'s group as $j$'s action affects $k$'s decision. Hence, we must update our groups so that $k$ and $j$ are grouped together, which means that their respective groups should be merged.

We do this by keeping track of the grouped agents in AnytimePIBT-R (line \ref{lines:anytime-vertex-edge-conflict}). However, we do not terminate early and instead continue the current AnytimePIBT-R call to make progress and encounter other potential agents. Then in AnytimePIBT after AnytimePIBT-R return (Alg \ref{alg:alg-anytime-pibt}, line \ref{line:main-anytime-helper}), AnytimePIBT adds a new found group (if it exists) and removes non-disjoint groups (e.g., the old groups that just got merged into the new group). This group (i.e., its corresponding agents) will be replanned later on.

\subsubsection{Diving Time Across Groups}
Given a set of groups to solve, a naive anytime implementation would allocate $T_{out}$ to each AnytimePIBT-R grouped call and terminate once the cumulative time exceeds $T_{out}$. However, this would prioritize earlier groups. Thus in Alg \ref{alg:alg-anytime-pibt} (line \ref{line:group-timeout}) an additional optimization is to allocate time based on the group's size. Other methods are possible but we found this to work well. 
% A small caveat is this assumes we know the time budget in advance, which works in our use case of having a fixed planning budget.

% We note that there are a variety of choices on what to do in AnytimePIBT-H when a new group is found (Alg \ref{alg:alg-pibt-helpers}, lines \ref{lines:anytime-group1,lines:anytime-group2}). It is possible to terminate AnytimePIBT-H early when a new group is found and immediately update groups, but this throws away the current DFS progress and might cause more grouping calls as it will find one interaction per DFS call.

\subsection{Theoretical Properties}
% Anytime PIBT is guaranteed to eventually find the optimal solution given enough time. Note that within a group and ignoring other agents reserving states (i.e. populating $RS$), Anytime PIBT searches through all possible agents' actions and will thus eventually find the optimal solution.

\begin{lemma} \label{lemma:group-optimal}
    Given a sufficiently large timeout $T_{out}$, AnytimePIBT-R finds the optimal single-step solution $\pi^{t+1}_{\forall i\in Gr.AoP}$ for agents in $Gr.AoP$ given fixed $\pi^{t:t+1}_{\forall i\notin Gr.AoP}$ that should not be conflicted with.
\end{lemma}
\begin{proof}
Anytime PIBT considers all valid transitions for each agent except those that conflict with previously planned agents (lines \ref{lines:anytime-vertex-edge-conflict}) or are pruned (lines \ref{lines:anytime-prune-f}). Transitions conflicting with previously planned agents (populated by $\pi^{t:t+1}_{i\notin Gr.AoP}$ as well as agents planned in the current recursive call) can be safely skipped. In respect to pruning, since $c(s^t_k,s^{t+1}_k) + h(s^{t+1}_k) \geq 0$ for all possible $s^t_k, s^{t+1}_k$ and agents $k$, if an intermediate call has $F_c \geq F_{best}$, it will never lead to a better solution and thus can be pruned. Therefore Anytime PIBT will eventually find the optimal single-step solution $\pi^{t+1}_{\forall i\in Gr.AoP}$ satisfying it does not conflict with $\pi^{t:t+1}_{\forall i\notin Gr.AoP}$.
\end{proof}

\begin{theorem} \label{thm:optimality}
    Given a sufficiently large timeout $T_{out}$, Anytime PIBT will eventually find the optimal single-step solution for all agents.
\end{theorem}
\begin{proof}

Suppose Anytime PIBT does not find the optimal single-step solution. Then there must exist some agent $k$ whose chosen action is worse than its action in the optimal solution. According to Lemma \ref{lemma:group-optimal}, AnytimePIBT-R will optimally plan agents in $Gr.AoP$ given fixed $\pi^{t:t+1}_{\forall i\notin Gr.AoP}$, so it must be the case that $k$'s optimal action is blocked by $\pi^{t:t+1}_{j}$ of some $j \notin Gr.AoP$.

% Given that $RS$ is populated by agents not in $k$'s disjoint agent group $Gr^k$, t
% This means that there is an agent $j=k$ or $k'$ in $Gr^k$ whose optimal action is blocked by another agent $j \notin Gr^k$. 
Since $Gr.AoP$ is defined by a disjoint agent group, this means that $k$ and $j$ are in different disjoint agent groups.
However, $j$ would be detected in AnytimePIBT-R (Alg \ref{alg:alg-anytime-helper}, line \ref{lines:anytime-vertex-edge-conflict}) and $k, j$ would be grouped together.
Then, Anytime PIBT would plan for this new group (Alg \ref{alg:alg-anytime-pibt}, line \ref{line:main-anytime-new-group}), which contradicts the earlier statement of $k$ and $j$ being in different disjoint agent groups. Thus, AnytimePIBT will eventually find the optimal single-step solution.
\end{proof}
% The optimal single-step MAPF solution for all agents is composed of optimal single-step MAPF solutions for each disjoint agent group. Additionally, AnytimePIBT-R optimally solves an agent group given $RS$ which is populated by all other agents' actions. Thus, the only way

% Given that we are decomposing Anytime PIBT calls into separate groups, the only way Anytime PIBT will not find an optimal global solution is when a grouped Anytime PIBT call has an agent that wants to try a move blocked by $RS$.

% Within a group and ignoring other agents reserving states (i.e. $RS$ is empty initially), Anytime PIBT without pruning would recurse through all possible agents' actions. Additionally, our pruning  and will thus eventually find the optimal solution for that group.

% Given that we are decomposing Anytime PIBT calls into separate groups, the only way Anytime PIBT will not find an optimal global solution is when the solution of one group $Gr^j$ impedes the solution of another group $Gr^k$. However, if this occurs then there must exist an agent $k \in Gr^k$ whose optimal action was considered in Anytime PIBT but blocked by $j \in Gr^j$. This would manifest in Alg \ref{alg:alg-anytime-helper}, line \ref{lines:anytime-reserved-collisions} and would subsequently be grouped (line \ref{lines:anytime-group1}). Anytime PIBT \ref{alg:alg-anytime-pibt} would then replan for this new group, eliminating this situation. Thus, Anytime PIBT will eventually find the optimal single-step solution.
% \end{proof}

\subsection{Anytime PIBT Tiebreak}
As we will see in the next section, Anytime PIBT actually has a worse success rate than PIBT as its optimal one-step action means that higher priority agents may not push away lower priority agents and can get stuck. Therefore, we implemented a variant where agents only consider their actions that have the best individual f-value, i.e., we search over tiebreaking between agent's best actions. We call this variant Anytime PIBT Tiebreak. This does not ensure global optimality but does retain regular PIBT properties (e.g., that the highest priority agent will make progress towards its goal at every step if the path to the goal is biconnected).

\section{Experiments} \label{sec:experiments}
We evaluate Anytime PIBT on using the standard MAPF benchmark maps \cite{stern2019mapfbenchmark} with varying numbers of agents and anytime cutoffs.
Since Anytime PIBT is a drop in replacement for PIBT, we also evaluate its effect in LaCAM and LaCAM* which both internally use PIBT. All instances were run with a 60 second cutoff time (i.e., sum of all planning times across all iterations).

Our experiments aimed to answer two questions. First, how much does Anytime PIBT improve single-step performance compared to PIBT? Second, how much does Anytime PIBT improve full horizon planning compared to using PIBT (e.g., by itself or in LaCAM, LaCAM*)?
Our results find that we perform quite well in finding single-step solutions and many times can find optimal single-step solutions. Interestingly, we see that this however does not translate to any meaningful full horizon solution quality gains.

\subsection{Single-Step Solutions} \label{sec:single-step}
% This section shows how Anytime PIBT is able to improve single-step performance compared to PIBT.
Figure \ref{fig:timestep-costs} shows an example run of Anytime PIBT with 1 second per-timestep cutoff on map den520d with 500 agents. At every timestep, we record and plot the initial normalized f-value (blue) and the final normalized f-value after Anytime PIBT finishes (orange). Note that the initial value (blue) is the result of the initial PIBT call. The normalized f-value subtracts out the lowerbound f-value (i.e., the sum of each agent's individual minimum f-value), so a value of zero means that all agents are going on their individual optimal action. 
We see that Anytime PIBT is able to consistently improve the f-value compared to the initial (PIBT) result.
% with larger improvements near the 50-150 timestep mark with the largest congestion.

Additionally, in this scenario, Anytime PIBT never reached the 1 second cutoff time, which means that it was able to find the optimal single-step solution for every timestep.

Figure \ref{fig:histogram} shows a histogram of the per-timestep across different Anytime PIBT deadlines in \textit{milliseconds} for 25 scenes on den520d with 500 agents. The x-axis denotes PIBT's f-value minus Anytime PIBT's f-value, e.g. a change of 4 means that Anytime PIBT improved (decreased) the f-value by 4 compared to PIBT's initial solution. We see that Anytime PIBT can make a non-trivial improvement with 0.1 milliseconds (red) and has more substantial single-step improvements within more time (purple, green).
We generally saw similar patterns in Figures \ref{fig:timestep-costs} and \ref{fig:histogram} across the different maps and agents with Anytime PIBT having larger impacts on more crowded instances. 
% In fact, in this scenario Anytime PIBT never reached the cutoff time of 1 second, which means that Anytime PIBT was able to find the optimal solution.

We found that grouping was a crucial part to Anytime PIBT's fast performance. Without grouping, Anytime PIBT's number of recursive calls increases exponentially with the number of agents such that in most cases the algorithm would not finish planning a single-step for 100 agents within 1 minute. This was due to the fact that without agent groups, finding the optimal action for an agent low in the priority list requires recursing through all the possible actions for all the agents above it in the priority list, which is highly inefficient if the higher agents are unrelated to the lower one. Hence, our group logic is critical for Anytime PIBT efficiency.

\begin{figure}[t!]
    \centering
    \includegraphics[width=0.45\textwidth]{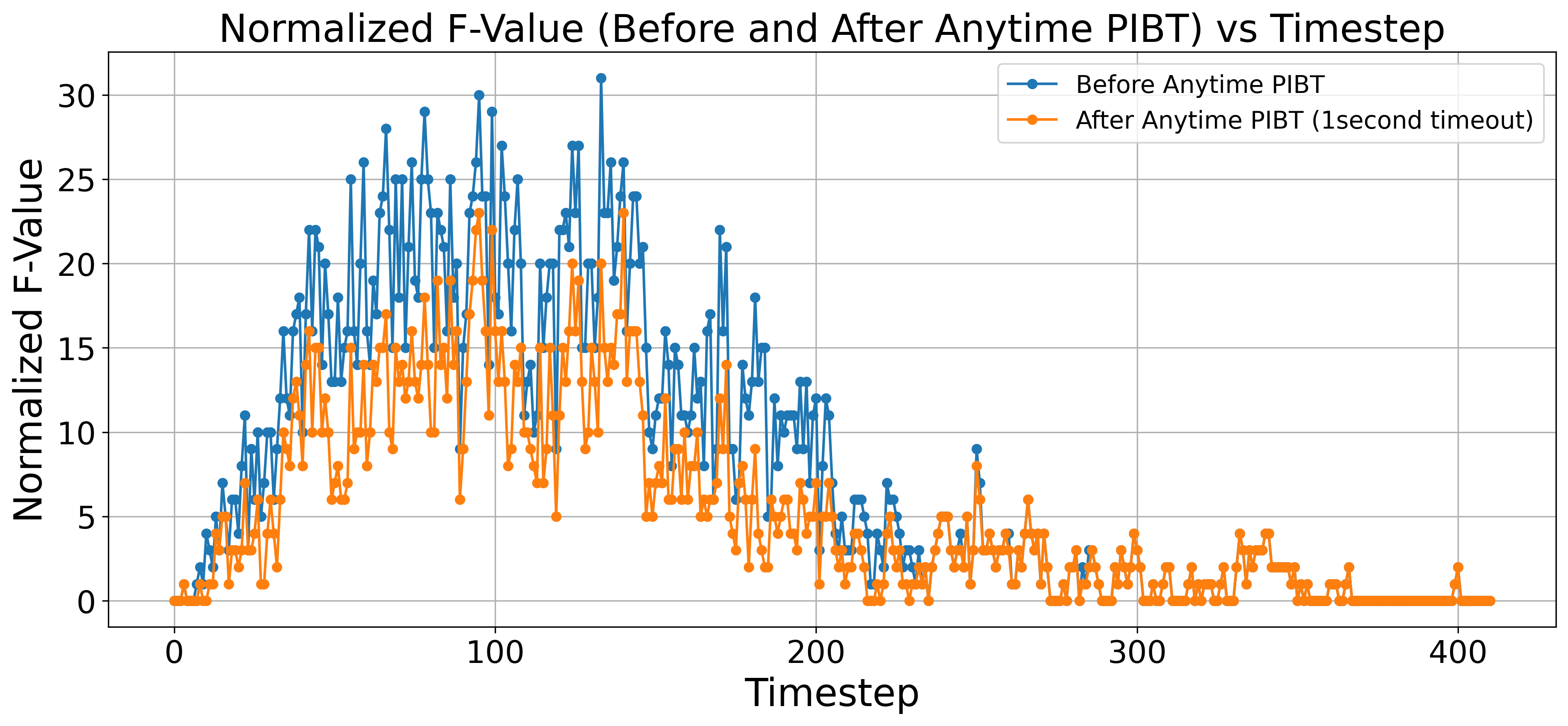}
    \vspace{-1em}
    \caption{We visualize the per timestep normalized f-values for 500 agents on map den520d. The normalized f-value is the solution f-value minus the lower-bound f-value (e.g. sum of each agent's best action). At every timestep, we plot the initial PIBT solution value (blue) and the solution after running Anytime PIBT for 1 second (orange).}
    \label{fig:timestep-costs}
    \vspace*{-0.5em}
\end{figure}

\begin{figure}[t!]
    \centering
    \includegraphics[width=0.45\textwidth]{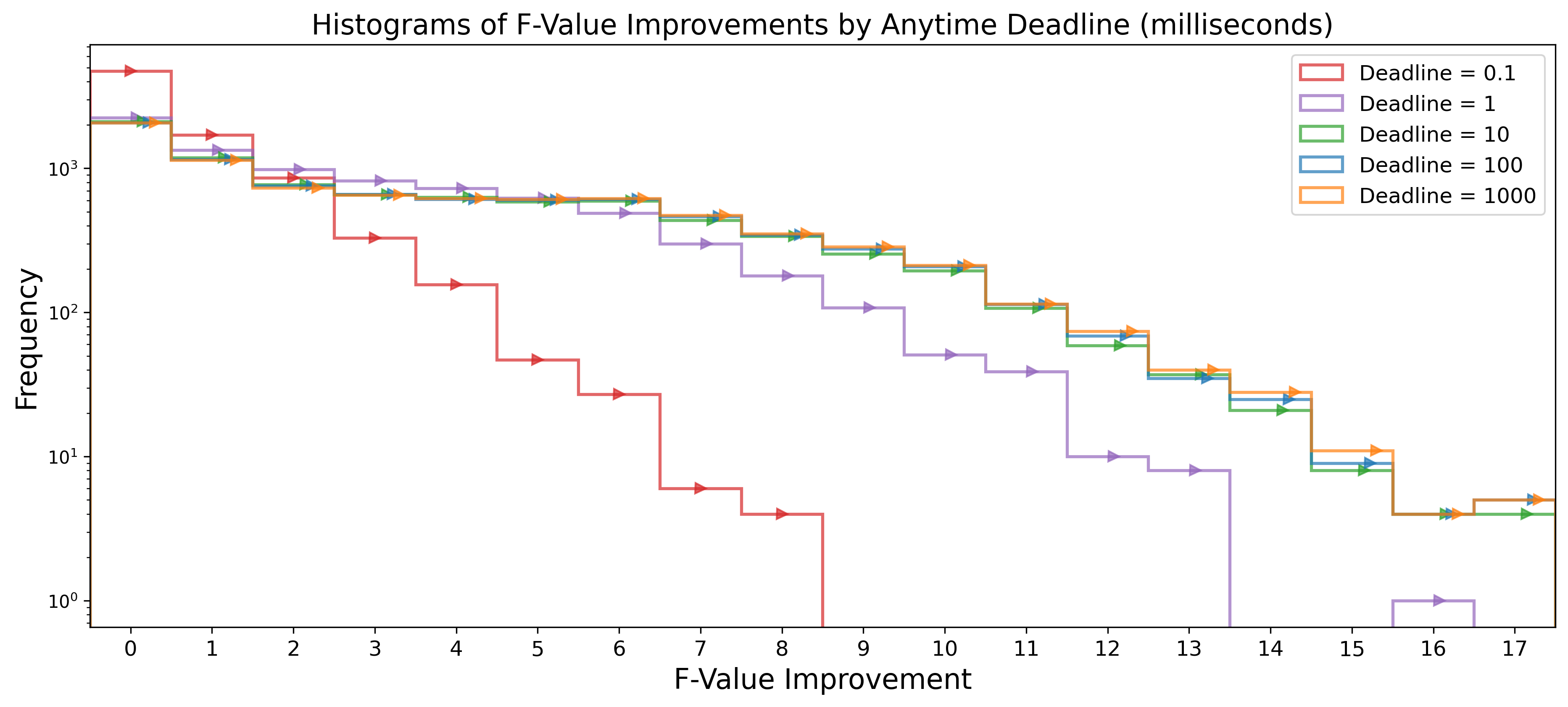}
    \vspace{-1em}
    \caption{We run Anytime PIBT and store the f-value improvements compared to PIBT's initial solution at different timeouts. Deadlines are in additional \textit{milliseconds} to the initial PIBT call (e.g., deadline of 0 is identical to PIBT).}
    \label{fig:histogram}
    \vspace*{-1em}
\end{figure}

% \begin{figure*}[t!]
%     \centering
%     \includegraphics[width=0.85\linewidth]{figs/main-results-fixed.png}
%     \caption{We compare PIBT and Anytime PIBT across a variety of maps and methods. Since LaCAM and LaCAM* use PIBT internally, we test the effect of replacing PIBT with Anytime PIBT. Overall the performance difference is surprisingly small.
%     % , with Anytime PIBT not having any significant improvement.
%     }
%     \label{fig:main-result}
%     \vspace*{-1em}
% \end{figure*}

\begin{figure*}[t!]  % 't' forces it to top of the next page
    \centering
    \begin{subfigure}{0.85\textwidth}  % Adjust width as needed
        \centering
        \includegraphics[width=\textwidth]{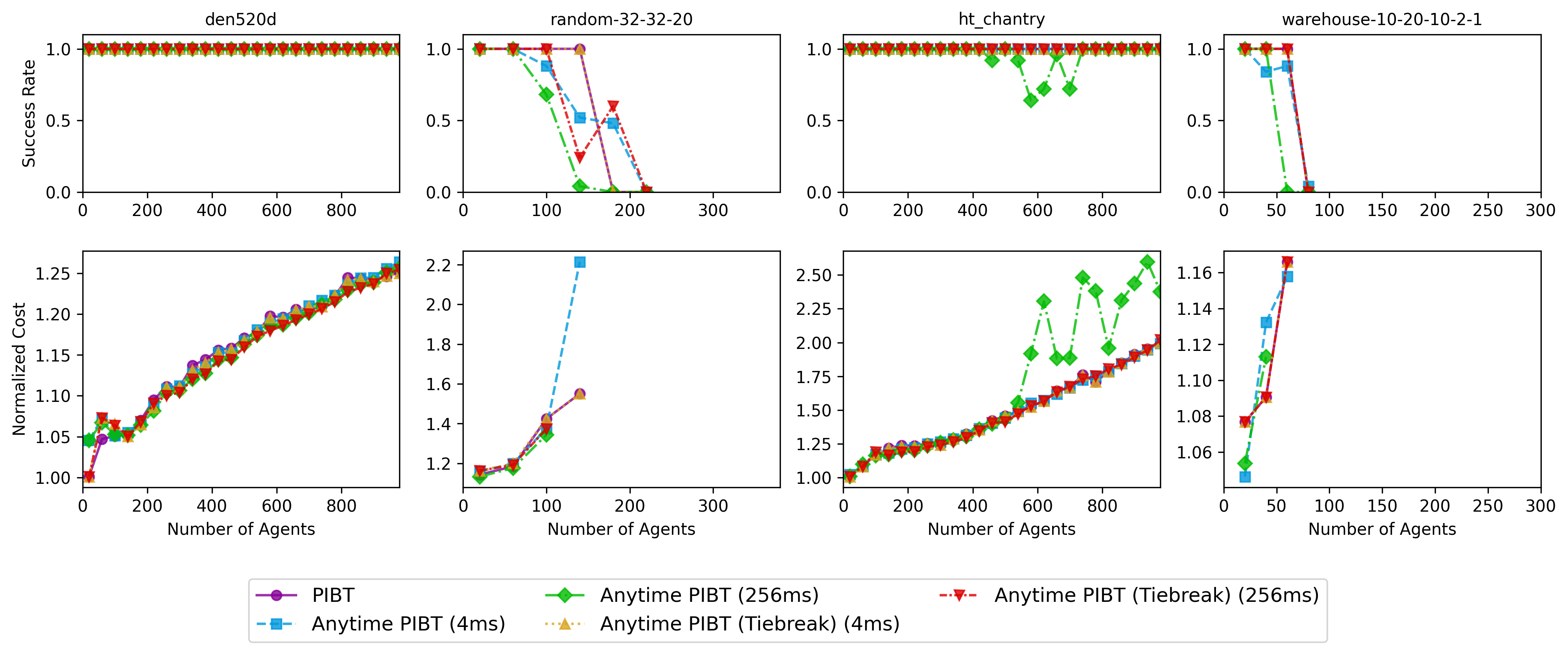}
        \caption{Standalone PIBT vs Anytime PIBT}
        \label{fig:pibt-results}
    \end{subfigure}
    
    % \vspace{5pt}  % Adjust vertical spacing

    \begin{subfigure}{0.85\textwidth}
        \centering
        \includegraphics[width=\textwidth]{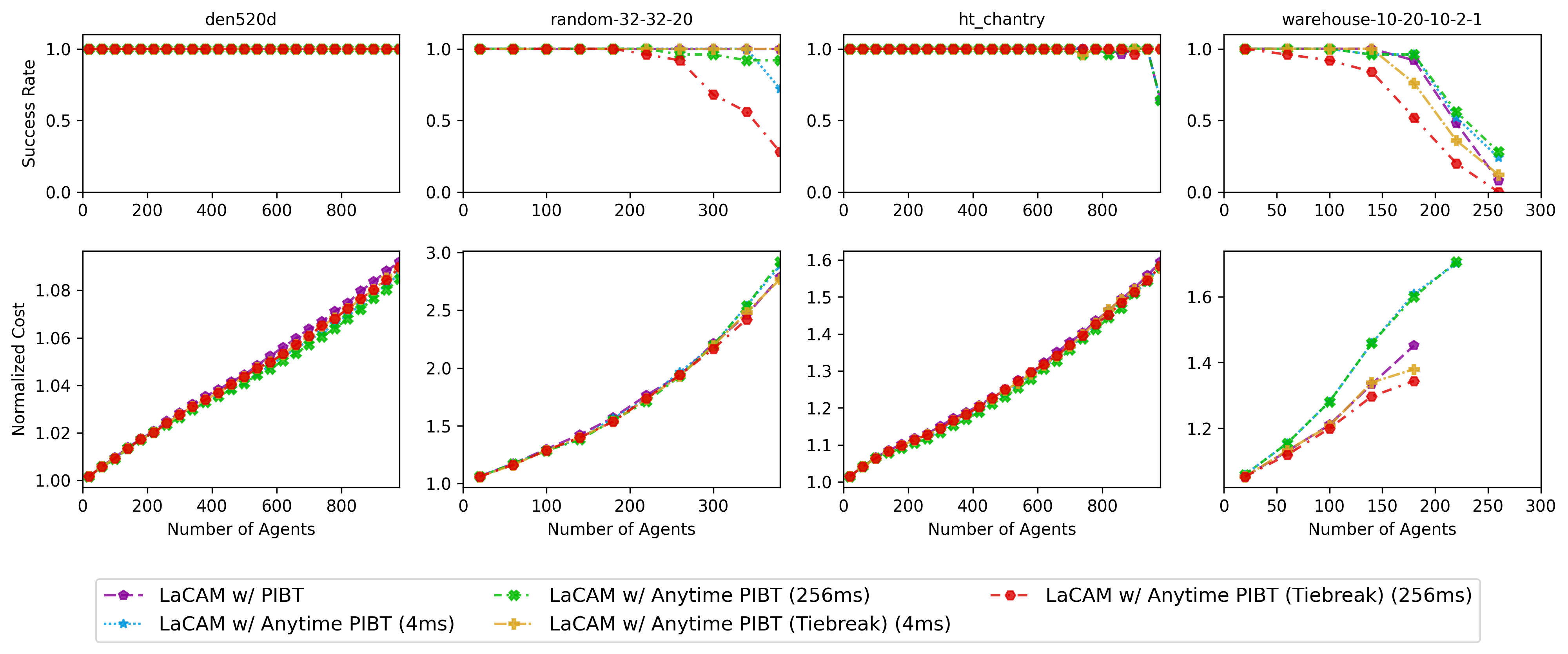}
        \caption{PIBT vs Anytime PIBT in LaCAM}
        \label{fig:lacam-results}
    \end{subfigure}
    
    % \vspace{5pt}

    \begin{subfigure}{0.85\textwidth}
        \centering
        \includegraphics[width=\textwidth]{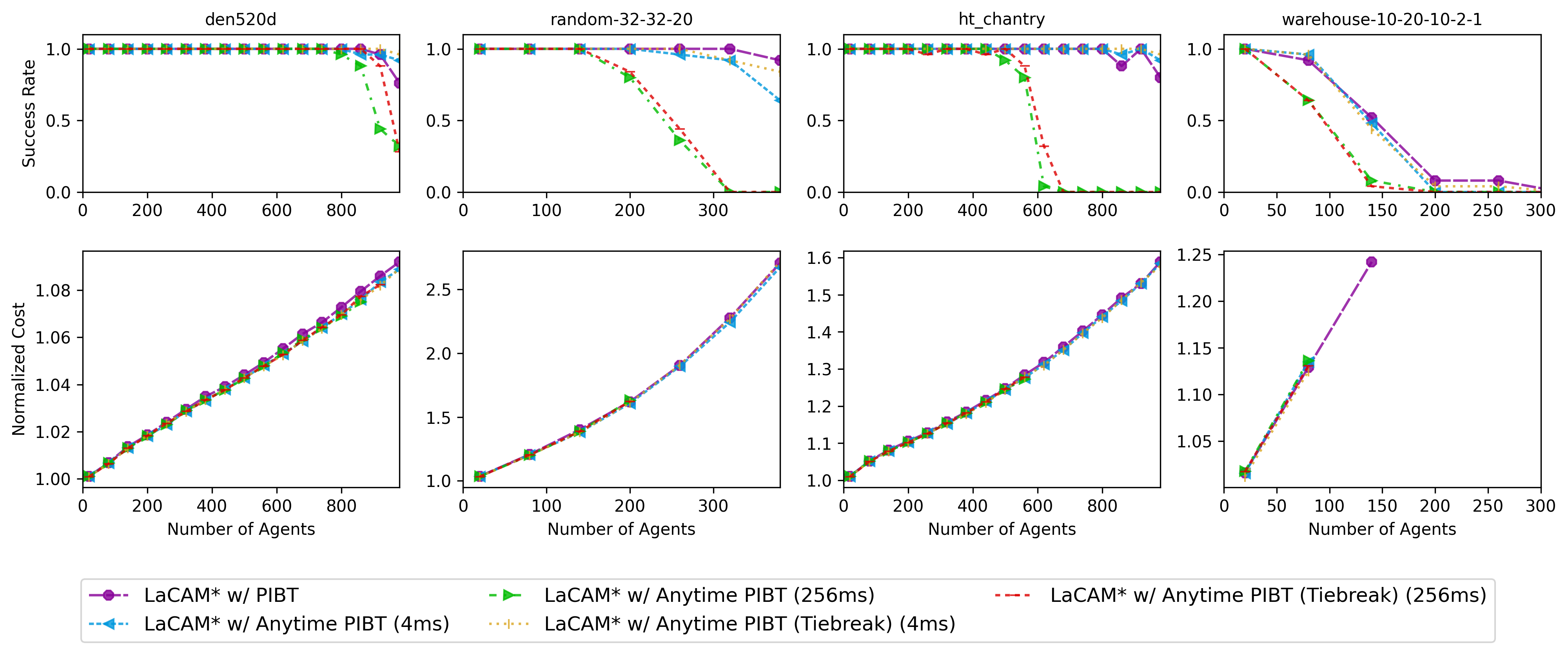}
        \caption{PIBT vs Anytime PIBT in LaCAM*}
        \label{fig:lacam-star-results}
    \end{subfigure}

    \caption{Evaluating PIBT vs Anytime PIBT}
    \label{fig:main-results}
\end{figure*}

\subsection{Full-Horizon Solutions}
Figure \ref{fig:main-results} shows the performance of Anytime PIBT vs PIBT across a variety of maps as standalone planners along as a configuration generator in LaCAM and LaCAM*. For each variant (standalone, in LaCAM, in LaCAM*), we compare using PIBT, normal Anytime PIBT with a 4 or 256 ms per-step cutoff, and Anytime PIBT Tiebreak with identical cutoffs. We plot the mean normalized cost (total cost divided by lower bound) for instances with a 50+\% success rate.

\subsubsection{PIBT vs Anytime PIBT}
Comparing PIBT and Anytime PIBT in Fig \ref{fig:pibt-results}, we see that Anytime PIBT has a perfect success rate on den520d and ht\_chantry, but worse performance than PIBT on random-32-32-20 and warehouse-10-20-10-2-1. This reveals that PIBT's greedy nature actually helps performance compared to Anytime PIBT. For example, when a low priority agent is blocking a high priority agent in a corridor, Anytime PIBT could find the optimal single-step solution which is for both agents to wait (in deadlock), whereas PIBT will greedily drive the higher priority agent forward. Anytime PIBT Tiebreak, on the other hand, does not suffer from this issue and has similar success rate as PIBT.

% with a small but consistent improvement in solution cost.
% This occurs because there are instances where the optimal action could require agents to stay at the same location. Anytime PIBT can find this solution and get stuck in deadlock while regular PIBT will instead prioritize a high priority agent. For example, when a low priority agent is blocking a high priority agent in a corridor, an optimal single-step solution is for both agents to wait, whereas PIBT will greedily drive the higher priority agent forward.

We additionally see that Anytime PIBT (with and without tiebreaking) is able to have small but noticeable improvement on overall solution quality. Even though this change is small, it is possible that this is the best that a single-step planner can do (without additional information) as the per timestep improvements seen in Fig \ref{fig:histogram} are on the order of 1-17 for 500 agents when the total per timestep cost is on the order of 500 (so we would expect only few percentage improvement based on this data).

\subsubsection{Anytime PIBT in LaCAM(*)}
We first focus on LaCAM results (Fig \ref{fig:lacam-results}). We observe that the solution quality are again slightly better with Anytime PIBT, especially with Tiebreaking. 
% slightly favorable for Anytime PIBT on den520d and ht\_chantry, for PIBT on random-32-32-20). 
However, performance is noticeable worse in warehouse-10-20-10-2-1. We hypothesize this occurs as more agents are stuck in deadlock and accumulate cost with Anytime PIBT as the action generator (which is reflected in the standalone Anytime PIBT performance). 
% Thus LaCAM with Anytime PIBT requires longer solution to resolve deadlock, which is seen by LaCAM requiring more high level nodes (not plotted).

We now focus on LaCAM* results (Fig \ref{fig:lacam-star-results}). Since LaCAM* is an anytime solver, using Anytime PIBT as a configuration generator means that we now tradeoff evaluating many ``poorer" quality high-level nodes or evaluating fewer ``higher" quality high-level nodes. For example, with a 60 second timeout and assuming the configuration generator takes a full 256 milliseconds per call, LaCAM* would generate at most 60/0.256 = 234 high-level nodes. LACAM* with a 4 millisecond timeout could generate 15,000 nodes instead. This discrepency explains why increasing the Anytime PIBT cutoff decreases the success rate in the more congested instances (e.g.,  ht\_chantry).
Looking closely at the cost, Anytime PIBT does decrease the cost slightly for the first three maps and marginally increases it for the last. However, overall the change in cost is small.

The overall lack of performance gain of Anytime PIBT in LaCAM and LaCAM* was quite unexpected, and implies a large gap between single-step solvers and the LaCAM methods. In particular, since LaCAM is very dependent on its single-step configuration generator, common intuition would think that improving the single-step action configuration would improve overall performance. Instead, even though Anytime PIBT consistently increases single-step performances over PIBT (Section \ref{sec:single-step}), this does not translate to gains with the LaCAM methods.

% We observe similar performance for all maps except for random-32-32-20. 
% \begin{figure*}[t]
%     \centering
%     \includegraphics[width=0.85\linewidth]{figs/anytime-plots.png}
%     \caption{We evaluate Anytime PIBT with cutoffs of 0, 1, 4, 16, 64, 256 additional milliseconds for PIBT (blue), in LaCAM (green), and in LaCAM* (red). Note x-axis in log scale. Blue: We observe that Anytime PIBT is largely flat after 1ms, implying that 1ms is sufficient for Anytime PIBT to sufficiently improve the solution. Green and Red: Anytime PIBT does not seem to consistently help or hurt here regardless of timeout.}
%     \label{fig:anytime-plots}
%     \vspace*{-1em}
% \end{figure*}

\section{Conclusion and Future Work}
PIBT is an extremely fast, scalable, and popular single-step solver. However, PIBT's speed comes at the expense of solution quality. An ideal anytime single-step solver would be as fast as PIBT but improve its solution over time. Anytime PIBT does exactly this.

Anytime PIBT interprets single-step MAPF as a recursive problem and therefore uses an anytime DFS to search through agents' actions. Anytime PIBT is able to save intermediate solutions and efficiently prune branches. Additionally, Anytime PIBT utilizes disjoint agent groups to decompose the entire recursive problem into smaller problems, and is able to dynamically update groups and eventually find the optimal single-step solution.

Experimentally, we find that Anytime PIBT is indeed able to improve single-step costs with milliseconds, and can even find the optimal single-step solution in many instances. However, counterintuitively these single-step improvements do not manifest themselves in full-horizon solutions when using Anytime PIBT inside LaCAM or LaCAM*. Still, we believe that Anytime PIBT is important as the first efficient anytime and eventually optimal single-step algorithm and see several avenues for future work.

\textbf{Anytime PIBT and OD:}
One interesting observation to OD \cite{standley2010operater_decomposition} is that Anytime PIBT could interpreted as an anytime version of OD that uses a depth first search and priority inheritance instead of an A* search. This perspective interprets a single recursive call of PIBT as generating an intermediate OD state moving only a single agent. PIBT itself is then searching using OD and a greedy depth first search, and is intelligently picking the next agent to expand using priority inheritance. 
Anytime PIBT's disjoint agent groups could also be interpreted as a version of OD's independence detection. The connection between OD and PIBT has not been described before and future work could benefit by further exploring this perspective. In particular, OD uses A*, PIBT uses DFS, while Anytime PIBT uses anytime DFS with pruning. It is possible that a different search algorithm using this perspective could be promising work.

\textbf{Integration with group cost functions:}
MAPF algorithms like CBS \cite{sharon2015cbs} and PIBT (and therefore methods that build off them like EECBS \cite{li2021eecbs} and LaCAM*) only natively support individual agent costs. Thus, incorporating a cost that is a function of the configuration of several agents (e.g., a formation cost) is not easily feasible in their frameworks. Since Anytime PIBT efficiently searches the joint configuration, it would be possible to modify Anytime PIBT to incorporate and search over these group costs.

% Since Anytime PIBT could incorporate group cost functions, can find optional single step solutions, and detects disjoint agent groups, Anytime PIBT could be modified to fit into the WinC-MAPF framework. This could significantly improve the poor success rates that Anytime PIBT exhibited.

\textbf{Single-Step Cost vs Full-Horizon:}
As seen in our experimental results, although we obtain single-step cost improvements, our overall full-horizon solution is not significantly better. Thus, broadly investigating how different single-step solvers and their properties affect full-horizon planning (e.g., in LaCAM) would be fruitful work.
Another avenue is to incorporate Anytime-PIBT with works such as \cite{chen2024traffic_flow} which explicitly compute heuristics online that incorporate congestion to avoid deadlock in the future.

Overall, Anytime PIBT is a promising single-step anytime algorithm with many potential future applications and extensions.

\bibliography{aaai25}

\clearpage

\appendix

\setcounter{figure}{0}
\renewcommand{\thefigure}{A\arabic{figure}}
\setcounter{table}{0}
\renewcommand{\thetable}{A\arabic{table}}

\section{Quick Summary}
\subsubsection{Recommended background readings:} Readers new to PIBT or LaCAM should read \citet{okumara2022pibt_jair} or \citet{okumura2023lacam} respectively. Readers new to Disjoint Agent Groups are recommended to read \citet{veerapaneni2024winc_mapf}.

\subsubsection{Motivation in respect to prior work:} PIBT is an extremely fast and effective single-step planner that is used in several other methods (e.g., LaCAM). However, PIBT is very greedy (which leads to poor solution quality) and cannot leverage extra planning time as it returns the first solution it finds. Ideally, we would have an ``anytime" version of PIBT that can leverage additional planning time to improve its solution cost. Our goal is to do exactly this.
% To that end, we design Anytime PIBT, which finds the same initial solution as PIBT but then uses additional planning time to improve costs.

\subsection{Intended Takeaways}
Our main contribution is designing Anytime PIBT, which finds the same initial solution as PIBT but then uses additional planning time to improve costs. We prove how Anytime PIBT will eventually find the optimal single-step solution, and empirically show how it can indeed do so even with hundreds of agents on certain instances. Anytime PIBT has a few algorithmic/theoretical insights and interesting experimental results which are our main takeaways:

1. Viewing single-step MAPF as a recursive process where planning for $N$ agents means we can plan for 1 one and get a smaller $N-1$ single-step MAPF instance. From this perspective, PIBT employs a DFS that returns the first solution it found. Thus, we can get anytime behavior in Anytime PIBT by employing an Anytime DFS.

2. The problem with the recursive perspective is that it grows exponentially with the number of agents, and thus performs poorly in practice. Our second key insight is that we can decompose our MAPF problem into groups of independent agents (formally defined as disjoint agent groups) and solve those separately. This makes Anytime PIBT run significantly faster. 

3. Anytime PIBT will provably find single-step optimal actions given sufficient time (and in practice can indeed find single-step optimal actions in less than a second).

4. Anytime PIBT has a consistent very small improvement in solution costs (e.g., less than 2 percent) when used for full horizon by itself or with LaCAM and LaCAM*. This is interesting as PIBT is typically thought of as a very greedy single-step planning, so our intuition is that improving the single-step solution cost would have a larger impact on overall costs. 

5. Finding optimal single-step paths can lead to poorer success rates as the optimal single-step solution as agents resting on their goal can become obstacles as pushing from their goal increases the cost. Another perspective is that PIBT's behavior of high priority agents pushing through low priority agents is long term beneficial despite it having high single-step solution costs. We therefore tested an Anytime-PIBT with tiebreaking that maintains priorities and only checks through agents best actions, which does result in a similar success rate to PIBT while still improving costs.

% \subsubsection{Limitations and Future works}
% The main limitation of the work is that the overall solution quality is not significantly impacted. This implies
% This works shows an initial promising single-step planner using PIBT, which is a very popular and scalable planner.

\end{document}